\documentclass{article}

\usepackage{arxiv}

\usepackage[utf8]{inputenc} 
\usepackage[T1]{fontenc}    
\usepackage{hyperref}       
\usepackage{url}            
\usepackage{booktabs}       
\usepackage{amsfonts}       
\usepackage{nicefrac}       
\usepackage{microtype}      
\usepackage{lipsum}
\usepackage{graphicx}
\graphicspath{{media/}}     
\usepackage{bbm}
\usepackage{amsmath}
\usepackage{amssymb}
\usepackage{amsthm}
\usepackage{mathtools}
\usepackage[colorinlistoftodos,prependcaption,textsize=small]{todonotes}
\usepackage{algorithm}
\usepackage{algpseudocode}
\usepackage{subcaption} 
\usepackage[title]{appendix}
\usepackage{multirow}
\usepackage{multicol}

\usepackage{natbib}
\usepackage{doi}

\usepackage{tikz}
\tikzset{>=latex} 
\usepackage{pgfplots} 
\usepackage{xcolor}
\usepackage[outline]{contour} 
\contourlength{1.2pt}
\usetikzlibrary{positioning,calc}
\usetikzlibrary{backgrounds}
\usetikzlibrary{arrows.meta, shadows} 

\tikzset{
	inplayer/.style={draw,minimum height=0.5cm,minimum width=4cm,anchor=center,rectangle,fill=white!20},  
	layer/.style={draw,minimum height=0.5cm,minimum width=4cm,anchor=center,rectangle,fill=blue!20},
	dropout/.style={draw,minimum height=0.5cm,minimum width=4cm,anchor=center,rectangle,fill=orange!20},
	norm/.style={draw,minimum height=0.5cm,minimum width=4cm,anchor=center,rectangle,fill=green!20},
	activation/.style={draw,minimum height=0.5cm,minimum width=4cm,anchor=center,rectangle,fill=red!20},
	output/.style={draw,minimum height=0.5cm,minimum width=4cm,anchor=center,rectangle,fill=purple!20},
	frame/.style={
		rectangle, draw,
		text width=6em, text centered,
		minimum height=4em,drop shadow,fill=white,
		rounded corners,
	},
	line/.style={
		draw, -{Latex},rounded corners=3mm,
	}, 
}

\pgfmathdeclarefunction{gauss}{3}{%
	\pgfmathparse{1/(#3*sqrt(2*pi))*exp(-((#1-#2)^2)/(2*#3^2))}%
}
\pgfmathdeclarefunction{cdf}{3}{%
	\pgfmathparse{1/(1+exp(-0.07056*((#1-#2)/#3)^3 - 1.5976*(#1-#2)/#3))}%
}
\pgfmathdeclarefunction{fq}{3}{%
	\pgfmathparse{1/(sqrt(2*pi*#1))*exp(-(sqrt(#1)-#2/#3)^2/2)}%
}
\pgfmathdeclarefunction{fq0}{1}{%
	\pgfmathparse{1/(sqrt(2*pi*#1))*exp(-#1/2))}%
}

\colorlet{mydarkblue}{blue!30!black}

\usepgfplotslibrary{fillbetween}
\usetikzlibrary{patterns}
\pgfplotsset{compat=1.12} 

\def\N{50}

\newtheorem{definition}{Definition}[section]
\newtheorem{theorem}{Theorem}[section]
\newtheorem{corollary}{Corollary}[theorem]

\newcommand\disteq{\mathrel{\stackrel{\makebox[0pt]{\mbox{\normalfont\tiny D}}}{=}}}

\title{Stochastic Actor-Critic: Mitigating Overestimation via Temporal Aleatoric Uncertainty}

\author{ \href{https://orcid.org/0000-0002-1406-6874}{\includegraphics[scale=0.06]{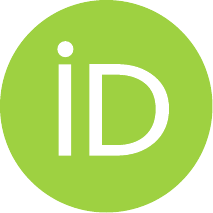}\hspace{1mm}Uğurcan Özalp} 	\\
	Turkish Aerospace \\
	Ankara, Türkiye \\
	\texttt{ugurcanozalp06@gmail.com} \\
	\texttt{ugurcan.ozalp2@tai.com.tr} \\	
}

\hypersetup{
pdftitle={Stochastic Actor-Critic: Mitigating Overestimation via Temporal Aleatoric Uncertainty},
pdfsubject={ cs.LG, eess.SY, cs.NE, math.ST},
pdfauthor={Uğurcan Özalp},
pdfkeywords={Reinforcement Learning, Uncertainty, Overestimation, Pessimism, Actor-Critic, Dropout},
}

\begin{document}
\maketitle

\maketitle

\begin{abstract}
	
	Off-policy actor-critic methods in reinforcement learning train a critic with temporal-difference updates and use it as a learning signal for the policy (actor). This design typically achieves higher sample efficiency than purely on-policy methods. However, critic networks tend to overestimate value estimates systematically. This is often addressed by introducing a pessimistic bias based on uncertainty estimates. 
	Current methods employ ensembling to quantify the critic's \textit{epistemic uncertainty}—uncertainty due to limited data and model ambiguity—to scale pessimistic updates.
	In this work, we propose a new algorithm called Stochastic Actor-Critic (STAC) that incorporates \textit{temporal (one-step) aleatoric uncertainty}—uncertainty arising from stochastic transitions, rewards, and policy-induced variability in Bellman targets—to scale pessimistic bias in temporal-difference updates, rather than relying on epistemic uncertainty.
	STAC uses a single distributional critic network to model the temporal return uncertainty, and applies dropout to both the critic and actor networks for regularization. 
	Our results show that pessimism based on a distributional critic alone suffices to mitigate overestimation, and naturally leads to risk-averse behavior in stochastic environments. Introducing dropout further improves training stability and performance by means of regularization. With this design, STAC achieves improved computational efficiency using a single distributional critic network. 
	
\end{abstract}



\section{Introduction}

Actor-critic methods leverage off-policy samples to train critics, promising higher sample-efficient learning than on-policy algorithms. Despite this advantage, actor-critic agents often struggle with \emph{overestimation bias} in value function learning due to the joint effects of function approximation, temporal-difference learning, and off-policy sampling \citep{thrun2014issues, sutton2018reinforcement, van2018deep}, which can destabilize training. In this work, we revisit this problem from a fresh perspective, focusing on the use of uncertainty modeling in actor-critic architectures. 

Epistemic uncertainty indicates a lack of training data and is used by current actor-critic methods to scale pessimistic updates \citep{fujimoto2018addressing, haarnoja2018soft, kuznetsov2020controlling, moskovitz2021tactical, chen2021randomized, hiraoka2021dropout, zhang2024explorer}. These methods use critic ensembles for this kind of uncertainty modeling. However, pessimistic updates based on epistemic uncertainty in the critic may hinder exploration of the state-action space, contradicting the principle of \textit{optimism in the face of uncertainty} \citep{kocsis2006bandit, audibert2007tuning, azizzadenesheli2018efficient, ciosek2019better, o2023efficient, wu2023uncertainty}.

On the other hand, aleatoric uncertainty is a measure of noise within training data. Some works propose to use pessimistic updates based on both epistemic and aleatoric uncertainty of return (discounted cumulative reward) for overestimation mitigation \citep{kuznetsov2020controlling}. Originally, modeling aleatoric uncertainty is a topic of distributional reinforcement learning \citep{bellemare2017distributional, dabney2018distributional, kim2022efficient, duan2021distributional, duan2025distributional, ma2025dsac}, and mainly used to scale risk sensitivity of the agent \citep{tang2019worst, yang2021wcsac, theate2023risk, ma2025dsac}. While pessimistic policy updates lead to risk-averse behavior, optimistic updates result in risk-seeking behavior. Recently, in Q-learning setting, \citet{achab2023one} proposed to model temporal aleatoric distribution of return, in which uncertainty is induced by only one-step dynamics of the environment, rather than across all steps. This approach is proven to have convergence guarantees theoretically. 

In off-policy actor-critic methods, the temporal-difference target changes continuously as the policy is updated. Although policy improvement is deterministic given a critic, high-dimensional function approximation causes small critic errors to be selectively exploited in a non-stationary manner, inducing irreducible variability in value targets even in deterministic environments, which is a major source of overestimation. We argue that this variability is best captured as temporal aleatoric uncertainty in the critic output, as it aggregates all uncertainty sources contributing to overestimation: policy-induced effects and transition or reward stochasticity.

In this work, we show that temporal aleatoric uncertainty can be modeled by a distributional critic, and that pessimism applied to this uncertainty alone is sufficient to control overestimation without ensemble of critics. Based on this idea, we introduce a novel off-policy distributional actor-critic algorithm, Stochastic Actor-Critic (STAC), specifically modeling temporal (one-step) return uncertainty for critic, and pessimistic updates based on this distribution. This way, STAC eliminates the need for double/ensemble critics, reducing both computation and memory costs. In other words, STAC uses \textit{pessimism in the face of temporal aleatoric uncertainty} for overestimation mitigation.  

Our method differs from other methods using aleatoric uncertainty for risk aversion \citet{tang2019worst, yang2021wcsac, kim2022efficient, ma2025dsac}. These methods, incorporate pessimism on total return uncertainty only for actor updates, whereas STAC uses pessimism on temporal return uncertainty for both critic and actor updates. The main purpose is overestimation mitigation rather than risk-averse learning, but pessimism in STAC still yields to risk aversion, similar to fully distributional methods. 

Recent work of \citet{nauman2024overestimation} has also shown that network regularization has positive impact on overestimation by reducing overfitting. Ensemble based methods inherently regularized by using multiple critic networks. Since STAC uses a single critic network, it is sensitive to learning errors and overfitting. For this purpose, dropout \citep{srivastava2014dropout} and layer normalization \citep{ba2016layer} are employed in both actor and critic networks as regularization tools.

The implementation is very simple and can be obtained by introducing dropout to networks, modeling a distributional critic and defining a pessimistic learning objective upon Soft Actor-Critic algorithm \citep{haarnoja2018soft}. We conduct experiments on standard RL benchmarks to evaluate the performance of STAC compared to existing methods. Our results demonstrate the effectiveness of STAC in achieving competitive performance with state-of-the-art methods while requiring fewer computational resources (single critic), making it a promising approach for real-world RL applications. 

\section{Preliminaries}

\label{sec:preliminaries}

This section introduces the minimum background required to analyze uncertainty-induced overestimation in off-policy actor-critic methods. Throughout the paper, $\mathcal{P}(\Omega)$ denotes the set of all possible probability distributions on the set $\Omega$. $A \disteq B $ indicates that two random variables, $A$ and $B$, have identical probability laws. 

\subsection{Aleatoric and Epistemic Uncertainty}

Uncertainty in an estimate is mainly categorized as either \textit{aleatoric} or \textit{epistemic}  \citep{der2009aleatory, kendall2017uncertainties, gal2016uncertainty}. Epistemic uncertainty refers to uncertainty of model parameters ($\theta$) due to insufficient training data. Given the training data $\mathcal{D}$ and prior distribution $p(\theta)$ over the network parameters $\theta$, we can compute the posterior distribution over the parameters $p(\theta \mid \mathcal{D})$ using Bayesian inference. 

On the other hand, aleatoric uncertainty is induced by the inherent randomness within the data. Even with infinite data, it is unavoidable and cannot be reduced, because it is an intrinsic part of the process being modeled. In deep learning, we can model this by having a distributional network that outputs a probability distribution $p_{\theta}(y|x)$ conditioned on input $x$ \citep{lakshminarayanan2017simple, kendall2017uncertainties}. Given a dataset $\mathcal{D}$, the loss function for training the network can be derived from the negative log-likelihood: $\mathcal{L}_{\theta}(\mathcal{D}) = \mathbb{E}_{\{(x_i, y_i)\} \sim \mathcal{D}} [ -\log{p_{\theta}(y_i|x_i)} ] $. In reinforcement learning, aleatoric uncertainty naturally arises from stochastic rewards, transition dynamics, and policy-induced randomness in temporal-difference targets. 

\subsection{Maximum Entropy Actor-Critic Reinforcement Learning}

In reinforcement learning language, the agent lives in a Markov Decision Process (MDP) which is represented by a tuple $\mathcal{M} = (\mathcal{S}, \mathcal{A}, d_0, \tau, R)$, where $\mathcal{S}$ is the state space, $\mathcal{A}$ is the action space, $d_0 \in \mathcal{P}(\mathcal{S})$ is the initial state distribution, $\tau: \mathcal{S}\times\mathcal{A} \rightarrow \mathcal{P}(\mathcal{S})$ is the transition kernel and $R: \mathcal{S}\times\mathcal{A} \rightarrow \mathbb{R}$ is the reward function. 

The initial state is sampled first, $s_0 \sim d_0(\cdot)$. At time step $t$, being in state $s_t$; next state is obtained from the environment, $s_{t+1} \sim \tau(\cdot \mid s_t,a_t)$ depending on the taken action $a_t \sim \pi(\cdot \mid s_t)$. Finally, a reward is obtained, $r_t=R(s_t, a_t)$ from the reward function $R$. State-return and state-action return are defined respectively as follows,

\begin{equation}
	\label{eqn:value_defn}
	G^{\pi}(s) \disteq \sum_{t=0}^{\infty} \gamma^{t} \big( R(s_t, a_t) - \alpha \log \pi(a_t|s_t) \big), \quad s_0=s, 
\end{equation}

\begin{equation}
	\label{eqn:maxent_zreturn_defn}
	Z^{\pi}(s, a) \disteq R(s, a) + \sum_{t=1}^{\infty} \gamma^{t} \big( R(s_t, a_t) - \alpha \log \pi(a_t|s_t) \big) , \quad s_0=s, \quad  a_0=a,  
\end{equation}

The ultimate goal of the agent is to derive a policy $\pi: \mathcal{S} \rightarrow \mathcal{P}(\mathcal{A})$ to maximize discounted return $G^{\pi}$ with entropy bonus \citep{haarnoja2017reinforcement, haarnoja2018soft}. However, returns are random variables, and cannot be used as objective. Therefore, value ($V$) and action-value ($Q$) functions are defined as expectations of return over policy and transition dynamics, $V^{\pi}(s) = \mathbb{E}_{\pi,\tau} \big[ G^{\pi}(s) \big]$, $Q^{\pi}(s, a) = \mathbb{E}_{\pi,\tau} \big[ Z^{\pi}(s, a) \big]$. Actor-critic methods model state-action value function $Q$, and learning iterates between solving policy evaluation and policy improvement. 

Policy evaluation minimizes the temporal difference: $\big\| \mathcal{T}^{\pi} Q(s, a) - Q(s, a) \big\|$ by a gradient step, where $\mathcal{T}^{\pi} Q$ is the Bellman backup operator, and defined as, 

\begin{equation}
	\label{eqn:q_eval}
	\mathcal{T}^{\pi} Q(s,a) = R(s,a) + \gamma \mathbb{E}_{\substack{s'\sim\tau(\cdot \mid s,a) \\ a' \sim \pi(\cdot \mid s')}} \big[ Q(s',a') - \alpha \log \pi(a' \mid s') \big]. 
\end{equation}

Policy improvement maximizes value function $V^{\pi}(s)=\mathbb{E}_{a\sim\pi } \Big[ Q^{\pi}(s,a) - \alpha \log \pi(a|s) \Big]$ by a gradient step. After policy improvement and policy evaluation at step $k$, updated state-action value function at step $k+1$ is $Q^{k+1}(s,a) = \mathcal{T}^{*} Q^{k}(s,a)$, where $\mathcal{T}^{*}$ is the Bellman optimality operator (Equation 5 from \citet{haarnoja2017reinforcement}), 

\begin{equation}
	\label{eqn:q_optimal}
	\mathcal{T}^{*} Q(s,a) = R(s,a) + \gamma \mathbb{E}_{s'\sim\tau(\cdot \mid s,a)} \Big[ \Tilde{\alpha} \log \Big( \int_{\mathcal{A}} \exp({\Tilde{\alpha}}^{-1} {Q}(s',a')) da' \Big) \Big]. 
\end{equation}

where $\Tilde{\alpha} \in [ \alpha, \infty ) $ is defined to demonstrate the continuum of policy improvement, i.e., the fact that policy is updated by partially exploiting state-action value function at each iteration. As $\Tilde{\alpha} \rightarrow \infty$, $\mathcal{T}^{*}Q$ boils down to $\mathcal{T}^{\mathcal{U}}Q$, which represents state-action value for uniform policy. On the other hand, $\Tilde{\alpha} = \alpha$ represents updated state-action value after complete exploitation. Consequently, in deep learning setting, a higher learning rate of actor network results in lower $\Tilde{\alpha}$, closer to $\alpha$. 

\section{Distributional Reinforcement Learning and Overestimation}

\subsection{Distributional Maximum Entropy Actor-Critic}
Instead of learning state-action value function $Q$, state-action return can also be modeled as a random variable to identify possible consequences of a given policy \citep{bellemare2017distributional}. This also constitutes the methodology of distributional maximum entropy actor-critic methods \citep{duan2021distributional, ma2025dsac}. For this, state-action return distribution  $\mathcal{Z}\in\mathcal{P}(\mathbb{R}^{\mathcal{S}\times\mathcal{A}})$ is defined, where state-action returns are sampled from this distribution, $Z(s, a) \sim \mathcal{Z}(s, a)$. Recall that $Q(s,a) = \mathbb{E}_{\mathcal{Z}, \tau} [ Z(s,a) ]$. The corresponding Bellman backup operator $\mathcal{T}^{\pi}$ is defined as, 

\begin{equation}
	\label{eqn:z_eval}
	\mathcal{T}^{\pi} Z(s,a) \disteq R(s,a) + \gamma \big( Z(s',a') - \alpha \log \pi(a' \mid s') \big), \quad s'\sim\tau(\cdot \mid s,a), \quad a' \sim \pi(\cdot \mid s'), 
\end{equation}

To make notation easier, distributional Bellman backup $\mathcal{T}^{\pi}_{D}\mathcal{Z}$ is defined, where Bellman backup of sampled values are sampled from this distribution, i.e., $\mathcal{T}^{\pi} Z(s,a) \sim \mathcal{T}^{\pi}_{D}\mathcal{Z}(s, a)$. At each step $k$, the state-action return distribution is updated by minimizing Kullback-Leibler divergence from distributional Bellman backup: $D_{\mathrm{KL}} \big( \mathcal{T}^{\pi}_{D}\mathcal{Z}(s, a) || \mathcal{Z}(s, a) \big)$ by a gradient step. 

\paragraph{One-step Distributional Uncertainty} Unlike conventional distributional modeling, it is also possible to model state-action return with randomness only up to first step \citet{achab2020ranking, achab2023one}. In this approach, Bellman backup uses expectation (deterministic) of return of next state and action. This way, uncertainty due to environment stochasticity and actor-induced uncertainty (due to policy improvement) are modeled for only one step, which directly targets the source of critic overestimation. This distribution is denoted as $\mathcal{Q}$ instead of $\mathcal{Z}$. Therefore, $Q\sim\mathcal{Q}$ is a random variable unlike previous formulation. Bellman backup is defined similar as in Equation \ref{eqn:z_eval} but expectation of next-state value is used, 

\begin{equation}
	\label{eqn:q_evs_eval}
	\mathcal{T}^{\pi} Q(s,a) \disteq R(s,a) + \gamma \big( \mathbb{E}_{Q\sim\mathcal{Q}}[Q(s',a')] - \alpha \log \pi(a' \mid s') \big), \quad s'\sim\tau(\cdot \mid s,a), \quad a' \sim \pi(\cdot \mid s'). 
\end{equation}

At each step $k$, the state-action value distribution is updated by minimizing Kullback-Leibler divergence from distributional Bellman backup: $D_{\mathrm{KL}} \big( \mathcal{T}^{\pi}_{D}\mathcal{Q}(s, a) || \mathcal{Q}(s, a) \big)$ by a gradient step. Corresponding Bellman optimality operator is as follows, 

\begin{equation}
	\label{eqn:q_optimal}
	\mathcal{T}^{*} Q(s,a) \disteq R(s,a) + \gamma \mathbb{E}_{Q\sim\mathcal{Q}} \Big[ \Tilde{\alpha} \log \Big( \int_{\mathcal{A}} \exp({\Tilde{\alpha}}^{-1} {Q}(s',a')) da' \Big) \Big], \quad s'\sim\tau(\cdot \mid s,a). 
\end{equation}

This operator is proved to be a contraction in Q-learning by \citet{achab2023one}, using $\max$ instead of $\text{logsumexp}$ operator. This guarantees learning convergence in tabular setting. However, in function approximation, there are other phenomena like overestimation and overfitting. As a main focus of this work, we will discuss overestimation bounds of one-step Bellman optimality operator, and propose a method to mitigate it in the next section. 

\subsection{Quantifying Overestimation for Sub-Gaussian Critic Distributions}

Let $\mu(s,a) = \mathbb{E}_{\mathcal{Q}, \tau} [ Q(s,a) ]$, the Bellman optimality backup of the mean is not equal to the mean of the backup, $\mathcal{T}^{*} \mu(s,a) \neq \mathbb{E}_{\mathcal{Q}, \tau} [ \mathcal{T}^{*} Q(s,a) ]$ due to policy improvement, and the difference is the overestimation bias. In this section, we analyze the overestimation bias similar to \citet{chen2021randomized} and \citet{lan2020maxmin} but in the soft learning framework instead of discrete actions. We define the overestimation error as the difference between $\mathbb{E}_{\mathcal{Q}, \tau} [ \mathcal{T}^{*} Q(s,a) ]$ and average $\mathcal{T}^{*} \mu(s,a)$ as $\epsilon$, 
\begin{equation}
	\epsilon(s,a) = \mathbb{E}_{\mathcal{Q}, \tau} [ \mathcal{T}^{*} Q(s,a) ] - \mathcal{T}^{*} \mu(s,a). 
\end{equation}

In the ideal case, $\epsilon(s,a)$ should be zero if there is no overestimation, which is not the case due to critic uncertainty. To quantify it, we assume that critic distribution $\mathcal{Q}(s, a)$ is sub-Gaussian with variance proxy $\sigma^2(s,a)$ and mean $\mu(s,a)$. We adopt a sub-Gaussian assumption as it provides a mild and widely used concentration model that enables analytic overestimation bounds without restricting the critic to a specific parametric distribution.

\begin{definition}
	\label{defn:subgauss}
	A random variable $X \in \mathbb{R}$ with mean $\mu=\mathbb{E}[X]$ is called sub-Gaussian with variance proxy $\sigma^2$ if its moment generating function satisfies
	\begin{equation}
		\mathbb{E}[\exp(\lambda X)] \leq \exp \big( \lambda \mu + \frac{1}{2} \lambda^2 \sigma^2 \big), \quad \forall \lambda \in \mathbb{R}. 
	\end{equation}
\end{definition}

Under the sub-Gaussian assumption, Theorem \ref{thm:overestimation} quantifies how critic uncertainty propagates through the soft Bellman optimality operator and Corollary \ref{cor:pessimism} motivates a variance-dependent pessimistic correction.

\begin{theorem}[Overestimation quantification for sub-Gaussian critics]
	\label{thm:overestimation}
	Let given state-action value distribution $\mathcal{Q}\in\mathcal{P}(\mathbb{R}^{\mathcal{S}\times\mathcal{A}})$ be sub-Gaussian with mean $\mu(s,a)$ and variance proxy $\sigma^2(s,a)$ for all state-action pairs, with bounded support. For each sample $Q\sim\mathcal{Q}$,  
	\begin{equation}
		\mathcal{T}^{*} Q(s,a) \leq R(s,a) + \gamma \Tilde{\alpha} \log \Big( \int\limits_{\mathcal{A}} \exp( {\Tilde{\alpha}}^{-1} \mu(s',a') + \frac{1}{2} {\Tilde{\alpha}}^{-2} \sigma^2(s',a') ) da' \Big). \\
	\end{equation}
\end{theorem}

\begin{corollary}[Pessimistic critic target]
	\label{cor:pessimism}
	For policy evaluation, using pessimistic Bellman backup, 
	\begin{equation}
		\label{eqn:q_evs_eval_pessimistic}
		\mathcal{T}_{\beta}^{\pi} Q(s,a) \disteq R(s,a) + \gamma \big( \mu(s',a') - \beta \sigma(s',a') - \alpha \log \pi(a' \mid s') \big), \quad s'\sim\tau(\cdot \mid s,a), \quad a' \sim \pi(\cdot \mid s'), 
	\end{equation}
	prevents overestimation as long as $\beta \geq \max\limits_{(s, a)} \frac{1}{2}{\Tilde{\alpha}}^{-1} \sigma(s, a)$. 
\end{corollary}

Moreover, maximum overestimation bound is demonstrated in Theorem \ref{thm:overestimation_bound}.

\begin{theorem}[Overestimation bound]
	\label{thm:overestimation_bound}
	Overestimation due to uncertainty of distribution $\mathcal{Q}$, denoted as $\epsilon$, is upper bounded for Bellman updates, 
	\begin{align}
		\epsilon(s,a) \leq \frac{\gamma}{2\Tilde{\alpha}} \mathbb{E}_{s'\sim\tau} \Big[ \max_{a'} \sigma^2(s',a') \Big]. 
	\end{align} 
\end{theorem}

As an additional notation, we introduce pessimistic distributional Bellman backup $\mathcal{T}_{\beta, D}^{\pi}\mathcal{Q}$ where pessimistic Bellman backup of sampled values are sampled from this distribution, i.e.,  $\mathcal{T}_{\beta}^{\pi} Q(s,a) \sim \mathcal{T}^{\pi}_{\beta, D}\mathcal{Q}(s, a)$. The difference between $\mathcal{T}^{\pi}_{\beta, D}\mathcal{Q}(s, a)$ and $\mathcal{T}^{\pi}_{D}\mathcal{Q}(s, a)$ is illustrated in Figure \ref{fig:evol_bellman_backup}. 

Taken together, these results show that overestimation in off-policy actor-critic methods originates from temporal uncertainty in Bellman targets, and that a simple variance-proxy-based pessimistic shift is sufficient to control it. This observation motivates the Stochastic Actor-Critic algorithm introduced next.

\section{Stochastic Actor-Critic}
\label{sec:stac}

In this section, we discuss key mechanisms needed for computational and sample efficient actor-critic learning and propose our algorithm \textit{Stochastic Actor-Critic}. This algorithm employs a \textit{single distributional} critic network $\mathcal{Q}_{\theta}$ by parameter set $\theta$ which captures temporal (one-step) aleatoric uncertainty, and dropout along with layer normalization for network regularization. Actor network $\pi_{\phi}$ outputs a \verb|tanh| transformed normal distribution to bound actions to $[-1, 1]$, and is parameterized by set $\phi$. Unlike other methods, STAC uses critic prediction in a pessimistic manner using temporal aleatoric uncertainty, for policy evaluation and policy improvement. 

The rationale behind this argument is that most of the temporal difference target randomness due to environment stochasticity and ongoing actor updates (policy improvement) inherently appears in the form of aleatoric uncertainty from the critic's side, and overestimation is the result of this randomness. Although policy improvement itself is deterministic given a critic, small approximation errors in the critic are selectively exploited by the policy update in a non-stationary manner, which induces unpredictable variability in Bellman targets in practice. Unlike distributional methods \citep{bellemare2017distributional}, we do not model aleatoric uncertainty of the full return, since it is not directly related to overestimation.

Although return uncertainty is modeled only for one-step, pessimism in STAC also leads to risk-averse learning since critic is also learned by pessimistic updates. Conventional distributional algorithms model full return as a distribution. For risk-averse learning only actor is updated in pessimistic way, not critic, usually by conditional value at risk measure (CVaR) \citep{dabney2018implicit, tang2019worst, yang2021wcsac, kim2022efficient, ma2025dsac} or by standard deviation \citep{ma2025dsac}. 

Epistemic pessimism naturally decreases the degree of overestimation by decreasing probability of critic errors which are exploited by policy improvement later. However, out-of-distribution states and actions are detected by high epistemic uncertainty, and should not be suppressed by pessimism. On the contrary, these state-action pairs should be visited for exploration. Therefore, epistemic uncertainty should not be used as a pessimism signal for overestimation mitigation in our opinion. 

\citep{nauman2024overestimation} demonstrated that network regularization methods improve performance. Since STAC employs a single critic network, it is prone to overfitting more than ensemble based methods. To prevent it, STAC employs dropout \citep{srivastava2014dropout} similar to \citet{hiraoka2021dropout}. While dropout may also promote exploration by injecting stochasticity into the policy and value estimates, we do not attempt to isolate or prove this effect here. Dropout also allows capturing the probabilistic nature of a network, representing Bayesian neural networks \citep{Gal2016Dropout}. However, it is important to note that STAC employs dropout on both critic and actor networks only for regularization purposes, not for epistemic modeling. 

\begin{figure}
	\centering
	\begin{tikzpicture}	
		\def\rew{-1.3}	
		\def\gam{0.85}
		\def\pess{1.2}
		\def\zmu{4.2};
		\def\zsig{1.9};
		\def\nzmu{9.0};
		\def\nzsig{1.5};
		\def\xmax{\nzmu+3.5*\nzsig};
		\def\xmin{\zmu-3.5*\zsig}
		\def\ymin{{-0.2*gauss(\nzmu,\nzmu,\nzsig)}};
		\def\ymax{{1.4*gauss(\nzmu,\nzmu,\nzsig)}}
		
		\begin{axis}[every axis plot post/.append style={
				mark=none,domain={-0.05*(\xmax)}:{1.08*\xmax},samples=\N,smooth},
			xmin=\xmin, xmax=\xmax,
			ymin=\ymin, ymax=\ymax,
			axis lines=middle,
			axis line style=thick,
			enlargelimits=upper, 
			ticks=none,
			xlabel=$Q$,
			every axis x label/.style={at={(current axis.right of origin)},anchor=north west},
			width=0.7*\textwidth, height=0.5*\textwidth,
			y=250pt
			]
			
			\addplot[blue, name path=B,thick] {gauss(x,\nzmu,\nzsig)};
			\addplot[teal, name path=B,thick] {gauss(x,\nzmu-\pess*\nzsig,\nzsig)};
			\addplot[red,  name path=S,thick] {gauss(x,\zmu,\zsig)};
			\addplot[black,dashed,thin]
			coordinates {(\nzmu,0) (\nzmu, {1.05*gauss(\nzmu,\nzmu,\nzsig)})}
			node[below=0pt,pos=0] {};			
			\addplot[black,dashed,thin]
			coordinates {(\nzmu-\pess*\nzsig,0) (\nzmu-\pess*\nzsig, {1.05*gauss(\nzmu,\nzmu,\nzsig)})}
			node[below=0pt,pos=0] {};
			
			\addplot[<-,black,thin,dashed]
			coordinates {(\nzmu-\pess*\nzsig,{1.01*gauss(\nzmu,\nzmu,\nzsig)}) (\nzmu, {1.02*gauss(\nzmu,\nzmu,\nzsig)})}
			node[above,at end] {$-\gamma \beta \sigma(s',a')$};
			
			
			\node[above=2pt,     black!12!blue] at (\nzmu+1.75*\pess*\nzsig,0.15) {$\mathcal{T}_{D}\mathcal{Q}(s,a)$};
			\node[above=2pt,     black!12!teal] at (\nzmu-2.5*\pess*\nzsig,0.2) {$\mathcal{T}_{\beta,D}\mathcal{Q}(s,a)$};									
			\node[above=2pt,     black!12!red]  at (\zmu-1.25*\pess*\zsig,0.1) {$\mathcal{Q}(s,a)$};
			
		\end{axis}
	\end{tikzpicture}
	\caption{Evolution of the Pessimistic Distributional Bellman Backup. While $\mathcal{T}^{\pi}_{D}\mathcal{Q}(s, a)$ is overestimated, a corrected backup $\mathcal{T}^{\pi}_{\beta, D}\mathcal{Q}(s, a)$ is closer to $\mathcal{Q}(s, a)$. }	
	\label{fig:evol_bellman_backup}
\end{figure}
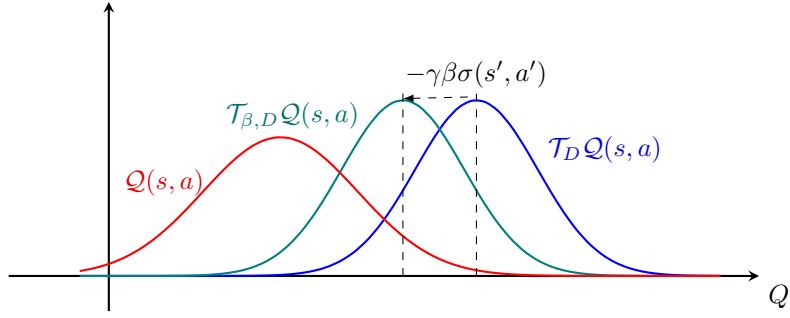

\subsection{Policy Evaluation}

For simplicity, STAC models critic as normal distribution. This contradicts the bounded distribution assumption of Theorem \ref{thm:overestimation}, but it yields a simple loss function and is easy to interpret. Still, it is reasonable to assume critic distribution is bounded for finite horizon or discounted MDPs ($\gamma<1$) with bounded reward functions. 

Using transition tuples from experience replay as batch, $\mathcal{D}_b = \{(s_i, a_i, r_i, s_i', \texttt{done}_i)\}_{i=1}^{N_b}$, temporal difference (TD) target $Q_i^{TD}$, representing Bellman backup, is $\beta$-pessimistic, 

\begin{equation}
	\label{eqn:td_target}
	Q_i^{TD} = r_i + \gamma (\mu_{\bar{\theta}}(s_i', \Tilde{a}_{i}') - \beta \sigma_{\bar{\theta}}(s_i', \Tilde{a}_{i}') - \alpha \log \pi_{\phi}(s_i', \Tilde{a}_{i}')) (\neg\texttt{done}_i), \quad \Tilde{a}_{i}'\sim\pi_{\phi}(\cdot \mid s_i'). 
\end{equation}

Learning objective is cross-entropy loss (log loss), 

\begin{equation}
	\mathcal{L}_{\theta}(\mathcal{D}_b) = \frac{1}{N_{b}} \sum_{i=1}^{N_b}  - \log{\mathcal{Q}_{\theta} (Q_i^{TD} \mid s_i,a_i)} = \frac{1}{2} \log{2\pi} + \frac{1}{2 N_{b}} \sum_{i=1}^{N_b}  \Big( \log{\sigma^2_{\theta}(s_i,a_i)} + \frac{(Q_i^{TD} - \mu_{\theta} (s_i,a_i))^2}{\sigma^2_{\theta}(s_i,a_i)} \Big). 
\end{equation}

\subsection{Policy Improvement}

According to Corollary \ref{cor:pessimism}, pessimistic TD targets should be used to train critic network. However, this analysis does not account for policy learning since policy is assumed as $\text{softmax}$ over state-action values. In actor-crtic framework, same pessimistic objective should be used for policy improvement, since policy evaluation and policy improvement objective should be the same. In other words, at learning step $k$, Bellman optimality must be equal to Bellman backup with the new policy, $\mathcal{T}^* Q^k \disteq \mathcal{T}^{\pi^{k+1}} Q^k$. It is only possible by using same pessimistic critic value for policy improvement. 

\subsection{Other Details and Algorithm Summary}

\paragraph{Layer Normalization} 
STAC implements Layer Normalization \citep{ba2016layer} after all hidden activations of critic and actor networks, in order to stabilize learning and prevent possible numerical instabilities. This has shown to improve performance significantly in deep reinforcement learning \citep{nauman2024overestimation, hiraoka2021dropout}. 

\paragraph{Lagged critic for TD target}
When the trained critic network is also used in calculating the target value, the critic training is prone to divergence \citep{li2023realistic}. For this, a common approach is to use another critic network to evaluate TD target \citep{mnih2013playing}. Similar to \citet{lillicrap2015continuous}, \citet{fujimoto2018addressing}, and \citet{haarnoja2018soft}, we use a delayed form of critic network for TD target evaluations as demonstrated in Equation \ref{eqn:td_target}. The parameters of target critic are only updated by Polyak averaging of main critic network weights through learning steps; $\bar{\theta} \leftarrow \rho \bar{\theta} + (1-\rho) \theta $. This strategy is important to ensure the stability of temporal difference learning. 

\paragraph{Automatic temperature tuning} Using constant temperature results in different policies if the reward magnitude changes. To mitigate this, \citet{haarnoja2018soft} proposed a policy entropy constraint, representing temperature as the Lagrange multiplier of the constraint. Given target entropy $\mathcal{\bar{H}}$ as hyper-parameter, the loss function related to this constraint is as follows;  
\begin{equation}
	\mathcal{L}_{\alpha}(\mathcal{D}_b) = - \alpha \mathcal{\bar{H}} + \alpha \sum_{i=1}^{N_b} \mathbb{E}_{a\sim\pi_{\phi}(\cdot \mid s_i)} \big[ -\log \pi_{\phi}(a \mid s_i) \big].  
\end{equation} 

Networks illustrations are available in Appendix \ref{app:network_archs}. Note that the bar notation stands for the lagged network with non-trainable parameters. STAC is summarized in Algorithm \ref{alg:stac} with gradient descent but Adam optimizer \citep{kingma2014adam} is used in our experiments. 

\begin{algorithm}
	\caption{Stochastic Actor-Critic}
	\label{alg:stac}
	\begin{algorithmic}
		\Require Environment $\texttt{env}$ 
		\Require Experience buffer $\mathcal{D}$ 
		\Require Critic $\mathcal{Q}_{\theta}$, lagged critic $\mathcal{Q}_{\bar{\theta}}$, actor $\pi_{\phi}$, all with dropout 
		\Require Initial temperature $\alpha$, target entropy $\Bar{\mathcal{H}}$  
		\Require Pessimism $\beta$ 
		\Require Learning rates $\eta_{Q}$, $\eta_{\pi}$, $\eta_{\alpha}$, Polyak parameter $\rho$
		\Require Total training steps $N$, batch size $N_b$
		\State $s \sim \texttt{env.reset}()$ \Comment{Reset the environment}
		\For{$N$ timesteps} 
		\State $a \sim \pi_{\phi}(\cdot \mid s)$  \Comment{Sample action}
		\State $r, s', \texttt{done} \sim \texttt{env.step}(a)$ \Comment{Act on environment}
		\State $\mathcal{D} \leftarrow \mathcal{D} \cup (s, a, r, s', \texttt{done})$  \Comment{Record transition tuple}
		\State \algorithmicif\ $\texttt{~done}$ \algorithmicthen\ $s \leftarrow s'$ \algorithmicelse\ $s \sim \texttt{env.reset}() $  \Comment{State transition or reset}
		\For{$G$ gradient steps}
		\State $\mathcal{D}_b = \{(s_i, a_i, r_i, s_i', \texttt{done}_i)\}_{i=1}^{N_b}  \sim  \mathcal{D} $  \Comment{Sample minibatch for training}
		\State $\Tilde{a}_{i}'\sim\pi_{\phi}(\cdot \mid s_i')$ \Comment{Sample next actions}
		\State $Q_i^{TD} = r_i + \gamma (\mu_{\bar{\theta}}(s_i', \Tilde{a}_{i}') - \beta \sigma_{\bar{\theta}}(s_i', \Tilde{a}_{i}') -\alpha \log \pi_{\phi}(\Tilde{a}_{i}' \mid s_i') ) (\neg\texttt{done}_i)$ \Comment{Build TD targets}
		\State $\theta \leftarrow \theta - \eta_{Q} \nabla_{\theta} \Big( \frac{1}{N_{b}} \sum_{i=1}^{N_b}  - \log{\mathcal{Q}_{\theta} (Q_i^{TD} \mid s_i,a_i)} \Big) $ \Comment{Update critic}
		\State $\phi \leftarrow \phi - \eta_{\pi} \nabla_{\phi} \Big( \frac{1}{N_{b}} \sum_{i=1}^{N_b} \mathbb{E}_{a\sim\pi_{\phi}(\cdot \mid s_i)} \big[ \mu_{\theta}(s_i, a) - \beta \sigma_{\theta}(s_i, a) - \alpha \log{\pi_{\phi}(a \mid s_i)} \big] \Big) $  \Comment{Update actor}
		\State $\alpha \leftarrow \alpha - \eta_{\alpha} \nabla_{\alpha} \Big( - \alpha \mathcal{\bar{H}} + \alpha \sum_{i=1}^{N_b} \mathbb{E}_{a\sim\pi_{\phi}(\cdot \mid s_i)} \big[ -\log \pi_{\phi}(a \mid s_i) \big] \Big) $  \Comment{Update temperature}
		\State $\bar{\theta} \leftarrow \rho \bar{\theta} + (1-\rho) \theta $  \Comment{Update target critic network}
		\EndFor
		\EndFor
		
	\end{algorithmic}
\end{algorithm}

\section{Experiments}
\label{sec:experiments}

Using the Gymnasium API \citep{towers_gymnasium_2023}, six MuJoCo and three Box2D environments are used for evaluation, as they are standard benchmarks in the literature. From Box2D, \verb|BipedalWalker-v3| and \verb|BipedalWalkerHardcore-v3| model a two-legged robot on randomly generated terrain, while \verb|LunarLander-v3| models a landing spaceship under wind and turbulence, which is maximized in our experiments. Therefore, these environments are inherently stochastic unlike MuJoCo environments. 


For evaluation, after each 1000 time steps, we execute a single test episode using the online policy and measure its performance by calculating the total reward accumulated during the episode. In a test episode, dropout and policy temperature are set to zero, following the best practice in the literature. Specified environments are trained through a fixed number of environment interactions, repeated with 5 seeds to assess the stability of the algorithm. Further experimental details are presented in Appendix \ref{app:hyperparam_exp}.

Hyper-parameters per environment can be found in Table \ref{tab:parameters_algos} of Appendix \ref{app:hyperparam_exp}. For all experiments, PyTorch (version 2.7.1) \citep{paszke2019pytorch} is used. Please refer to Appendix \ref{app:code} for the codebase. 

\subsection{Comparison to Other Algorithms}

Our experiments aim to investigate whether enhancing off-policy actor-critic methodology with STAC can improve their sample and computational efficiency on continuous-control benchmarks. For this purpose, STAC is compared to similar competitive algorithms; Distributional Soft-Actor Critic (DSAC) \citep{ma2025dsac}, and Soft Actor-Critic (SAC) \citep{haarnoja2018soft}. Both methods use minimum of double critics for updates. All algorithm results are obtained using in-house code with the same network architectures (including layer normalization but not dropout) to make a fair comparison. 

We also add another algorithm called Epistemic Stochastic Actor-Critic (ESTAC), with double distributional critic with one-step uncertainty similar to STAC. For overestimation, minimum of two critic is used for learning similar to SAC and DSAC, instead of temporal aleatoric pessimism. The purpose is to isolate the source of STAC's improvements by replacing temporal aleatoric pessimism with epistemic overestimation mitigation.

The pessimism rates and dropout configurations used for comparison are resulted from ablation studies, and explained in the next section. Optimal configurations are summarized in Table \ref{tab:target_ent_pessimism_algos}. 

\paragraph{Learning curves} In Figure \ref{fig:main_comparison}, the performance of STAC is shown against previously mentioned state-of-the-art algorithms. Additionally, value estimation errors are presented in Figure \ref{fig:error_comparison}. Value estimation error is measured as the difference between critic prediction and discounted return collected during evaluation episodes. The bold lines represent the inter-quartile mean across random seeds, while the shaded area indicates the corresponding quartile ranges of the total reward across different seeds. Curves are smoothed through time for better visibility. 

Final performane of learning processes averaged over random seeds are summarized in Table \ref{tab:env_iqm_results}. 


\begin{table}[h!]
	\addtolength{\tabcolsep}{1pt}
	\centering
	\caption{Inter-quartile mean results of last \%1 of evaluation episodes, $\beta$ and dropout for STAC is selected with best performance. } 
	\label{tab:env_iqm_results}
	\begin{tabular}{|c c c c c c |}
		\hline
		Env  & \# steps & DSAC & ESTAC & SAC & STAC \\ 
		\hline		
		\verb~Ant-v4~ & 3M & \textbf{7543.1} & 5798.9 & 7407.1 & 6907.3 \\
		\verb~BipedalWalker-v3~ & 1M & \textbf{335.3} & 319.2 & 326.3 & 308.5 \\
		\verb~BipedalWalkerHardcore-v3~ & 3M & 56.8 & 32.3 & -50.8 & \textbf{159.9} \\
		\verb~HalfCheetah-v4~ & 3M & \textbf{14921.5} & 13241.6 & 11475.6 & 14569.3 \\
		\verb~Hopper-v4~ & 1M & 1762.9 & 944.3 & 3196.4 & \textbf{3363.7} \\
		\verb~Humanoid-v4~ & 3M & \textbf{9484.1} & 8608.9 & 6990.4 & 8726.2 \\
		\verb~LunarLander-v3~ & 1M & 265.4 & 258.3 & \textbf{268.9} & 265.7 \\
		\verb~Swimmer-v4~ & 1M & 69.3 & 104.1 & 54.2 & \textbf{125.0} \\
		\verb~Walker2d-v4~ & 3M & 5981.4 & 5952.5 & 5036.4 & \textbf{6634.7} \\
		\hline
	\end{tabular}
\end{table}

\begin{figure}
	\centering
	\includegraphics[width=\textwidth,height=\textheight,keepaspectratio]{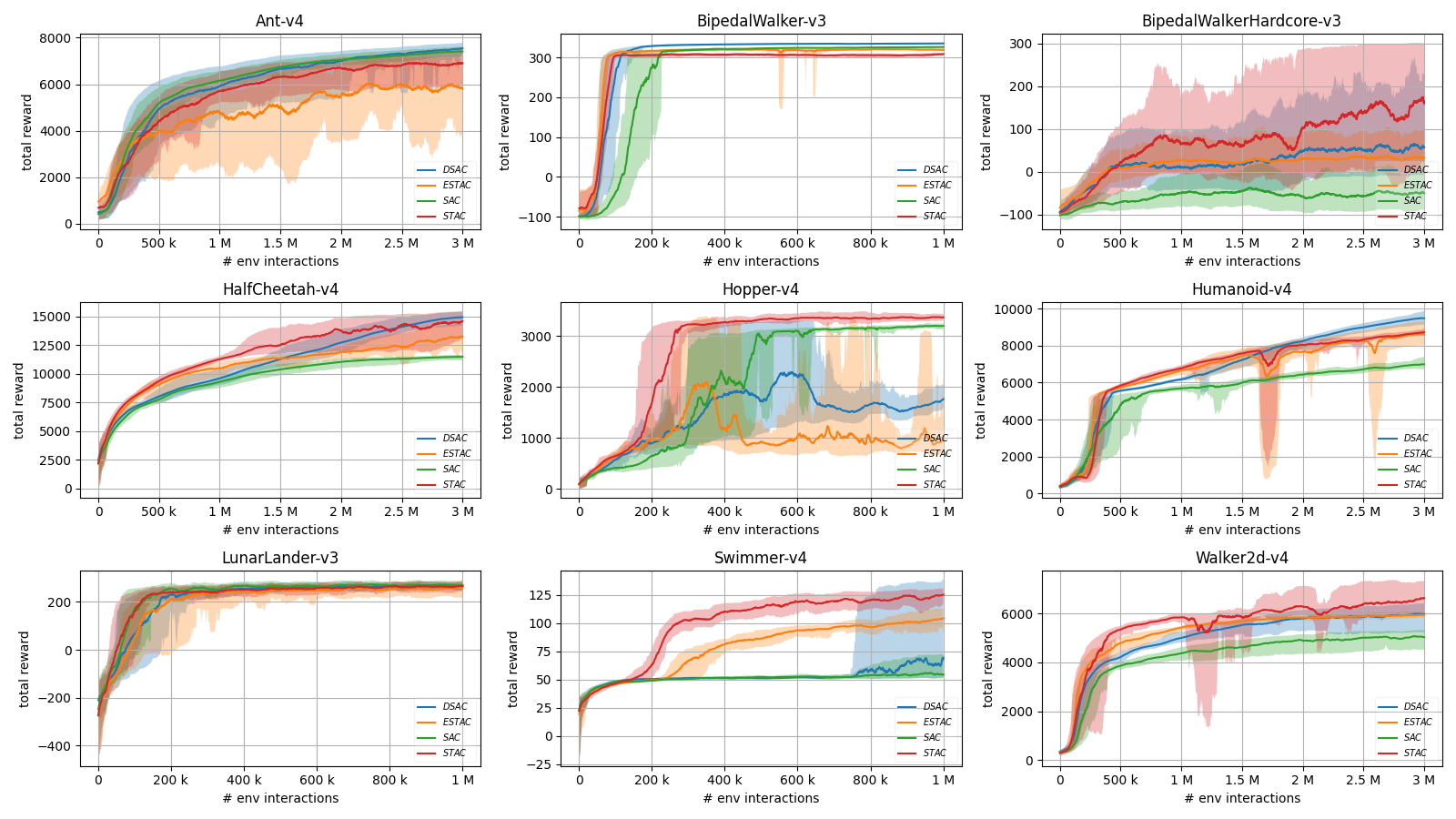}
	\caption{Episodic score curves of STAC and other algorithms. }
	\label{fig:main_comparison}
\end{figure}

\begin{figure}
	\centering
	\includegraphics[width=\textwidth,height=\textheight,keepaspectratio]{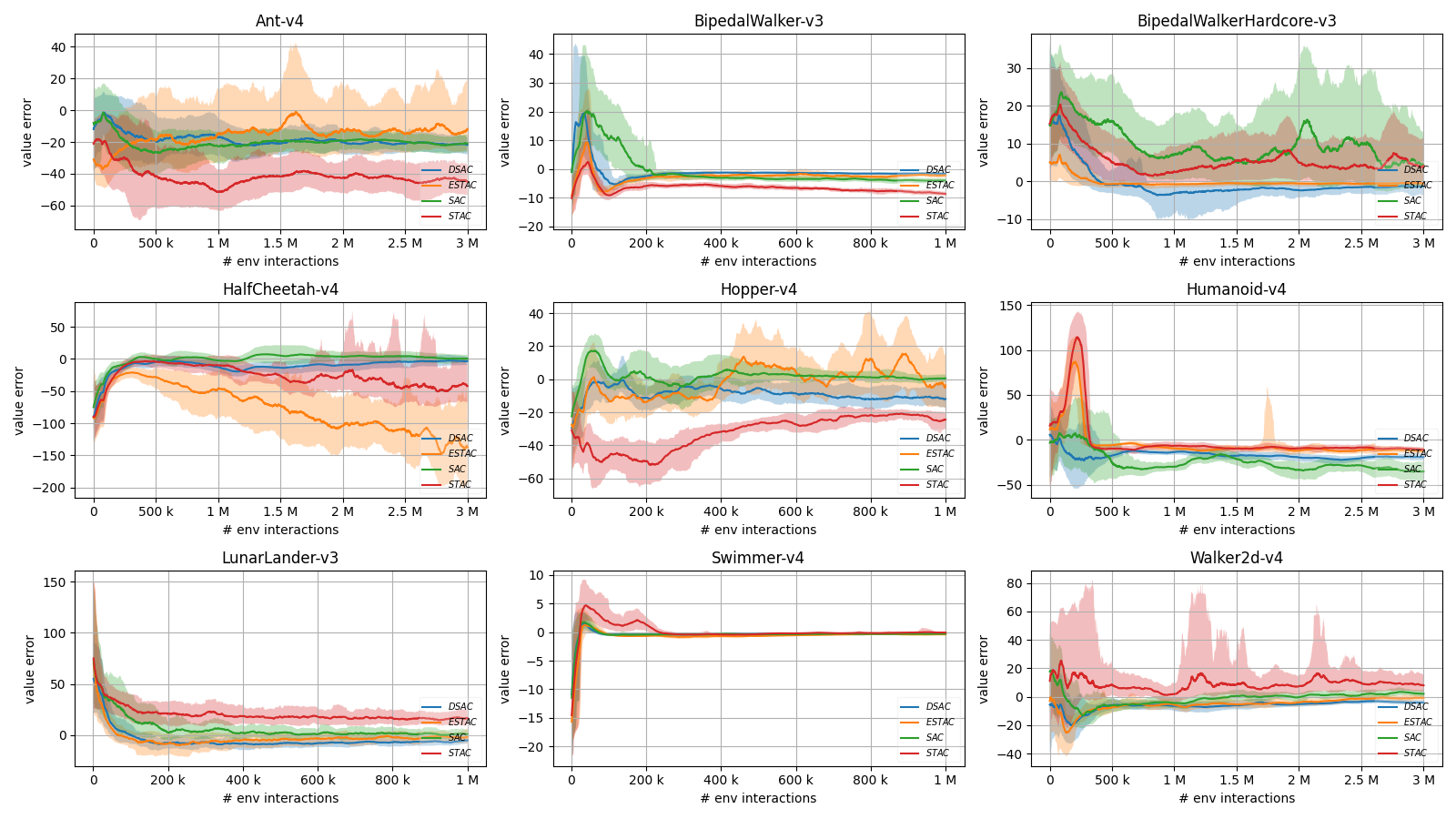}
	\caption{Average episodic value estimation error of STAC and other algorithms.}
	\label{fig:error_comparison}
\end{figure}

\paragraph{Sample efficiency} As seen from Figure \ref{fig:main_comparison} and Table \ref{tab:env_iqm_results}, STAC outperforms other algorithms on some environments in terms of sample efficiency (see \verb|BipedalWalkerHardcore-v3|, \verb|Hopper-v4|, \verb|Swimmer-v4|, \verb|Walker2d-v4|) while it fall behind in other environments. However, it is important to note that there are many factors that affect performance, such as distribution modeling method, number of critics and regularization method, network size and learning rates etc. Rather than absolute performance alone, we demonstrate that STAC achieves competitive or superior performance without convergence issues while using a single critic, whereas competing methods rely on double critics.

\paragraph{ESTAC vs STAC} Except on \verb|BipedalWalker-v3|, STAC consistently outperforms ESTAC across all environments. This indicates that temporal aleatoric pessimism is generally more effective than epistemic pessimism in the form of \textit{min-clipping}, at least in our experimental setting. 


\subsection{Sweeping Pessimism and Dropout}

STAC is experimented under varying levels of pessimism $\beta$ and dropout rates to isolate effect of pessimism and dropout individually. We used 5 pessimism level and 4 dropout configurations: no dropout, actor dropout, critic dropout and actor\&critic dropout. 

\begin{table}[h!]
	\addtolength{\tabcolsep}{1pt}
	\centering
	\caption{Inter-quartile mean scores of last \%1 of evaluation episodes, for varying $\beta$. Dropout rate is 0.01 for critic and actor. } 
	\label{tab:stac_ablation}
	\resizebox{\columnwidth}{!}{
		\begin{tabular}{|c c c c c c|}
			\hline
			Env & $\beta=0.0$ & $\beta=0.125$ & $\beta=0.25$ & $\beta=0.375$ & $\beta=0.5$ \\
			\hline
			\verb~Ant-v4~ & 5264.4 & \textbf{7082.7} & 6907.3 & 6246.4 & 4834.4 \\
			\cline{1-6}
			\verb~BipedalWalker-v3~ & 171.9 & 125.4 & 262.6 & \textbf{308.0} & 305.1 \\
			\cline{1-6}
			\verb~BipedalWalkerHardcore-v3~ & 61.8 & \textbf{118.8} & 43.4 & -0.0 & -59.8 \\
			\cline{1-6}
			\verb~HalfCheetah-v4~ & \textbf{14489.5} & 12299.6 & 13385.4 & 11042.1 & 7602.7 \\
			\cline{1-6}
			\verb~Hopper-v4~ & 1132.4 & 1308.8 & 2643.4 & 3318.2 & \textbf{3363.7} \\
			\cline{1-6}
			\verb~Humanoid-v4~ & 5623.9 & \textbf{8407.2} & 8277.9 & 7587.2 & 6930.8 \\
			\cline{1-6}
			\verb~LunarLander-v3~ & \textbf{265.7} & 263.8 & 157.2 & 14.3 & -77.7 \\
			\cline{1-6}
			\verb~Swimmer-v4~ & \textbf{103.7} & 93.5 & 70.0 & 54.2 & 45.9 \\
			\cline{1-6}
			\verb~Walker2d-v4~ & 4299.2 & \textbf{6306.4} & 5842.1 & 5795.2 & 5539.1 \\
			\cline{1-6}
			\hline
		\end{tabular}
	}
\end{table}


\paragraph{Learning curves} 

In the experiments, we presented controlled study to isolate dropout's role: while distributional pessimism eliminates overestimation bias, dropout primarily increases training stability and performance. Both learning curves are available in Appendix \ref{app:results_of_ablations}. Episodic scores and related value estimation errors for varying level of pessimism under joint actor/critic dropout are summarized in Figure \ref{fig:pess_sens_pol_cri_dropout} and Figure \ref{fig:pess_sens_error_pol_cri_dropout}.  In addition, same curves for different dropout configurations under same level of pessimism are summarized in Figure \ref{fig:drop_sens} and Figure \ref{fig:drop_sens_error}, where $\beta$ values are in Table \ref{tab:target_ent_pessimism_algos}. The shaded area represents quartile limits, while the solid line represents inter-quartile mean across different seeds and \%1 smoothing window of evaluation episodes through environment steps. 

\paragraph{Pessimism} Table \ref{tab:stac_ablation} summarizes final performance in terms of inter-quartile mean score on last \%1 evaluation episodes. Results indicate that $\beta$ is a sensitive parameter, the higher $\beta$ yields a negative error (see Figure \ref{fig:pess_sens_error_pol_cri_dropout}), consistent with the intended effect. In addition, score curves are worse if value estimations tend to be positive, meaning that critic overestimation is not mitigated enough (see \verb|Ant-v4|, \verb|BipedalWalker-v3|, \verb|Hopper-v4|, \verb|Humanoid-v4|, \verb|Walker-v4|). On the other hand, score curves are again worse when error curves are negative than it should be, meaning that the agent is stuck on critic underestimation caused by high pessimism (see \verb|BipedalWalkerHardcore-v3|, \verb|HalfCheetah-v4|, \verb|LunarLander-v3|, \verb|Swimmer-v4|). In the end, this sensitivity varies for different environments, possibly due to differences in reward sparsity and stochasticity. This parameter stands as the major bottleneck of STAC and can only be determined by this heuristic for now. 

\paragraph{Dropout}
Dropout's primary role in STAC is as a regularization tool to improve optimization stability and robustness to noisy data. In Figure \ref{fig:drop_sens}, independent of critic dropout, actor dropout consistently improves performance by stabilizing policy updates. According to Figure \ref{fig:drop_sens_error}, the overestimation trends observed with and without dropout are qualitatively similar, confirming that our core contribution: (mitigating overestimation via aleatoric uncertainty) remains effective independent of dropout, except \verb|Hopper-v4|. However, in this environment, learning does not even converge without critic dropout. Therefore, it is necessary in some environments for training stability rather than overestimation mitigation. Double/ensemble critic approaches mitigate training instability naturally, whereas STAC requires a convenient dropout in some cases. 

\paragraph{Pessimism under Environment Stochasticity}

Pessimistic learning based on aleatoric uncertainty can interact with exploration, especially in environments with complex dynamics. Comparing \verb|BipedalWalker-v3| and \verb|BipedalWalkerHardcore-v3|, we observe that lower pessimism performs better in the harder environment, while moderate pessimism improves performance in the simpler setting (see Figure \ref{fig:pess_sens_pol_cri_dropout} and \ref{fig:pess_sens_error_pol_cri_dropout}). One possible explanation is that, in simpler tasks, uncertainty is dominated by approximation and bootstrapping effects, where pessimism helps reduce overestimation. In contrast, in more challenging environments, uncertainty may be driven more strongly by the environment itself, and excessive pessimism may slow learning by preventing exploration. These observations suggest that the appropriate level of pessimism depends on the interaction between task complexity and learning dynamics.

\subsection{Risk Sensitivity by Pessimism}

In STAC, actor updates are pessimistic for one-step return uncertainty. However, critic targets are also pessimistic with respect to consecutive steps. Therefore, pessimism is also bootstrapped along with reward, and increasing risk-averse behavior is expected as $\beta$ increases. For this purpose, a toy problem is used to demonstrate this behavior. 

\begin{figure}
	\centering
	\includegraphics[width=\textwidth,height=\textheight,keepaspectratio]{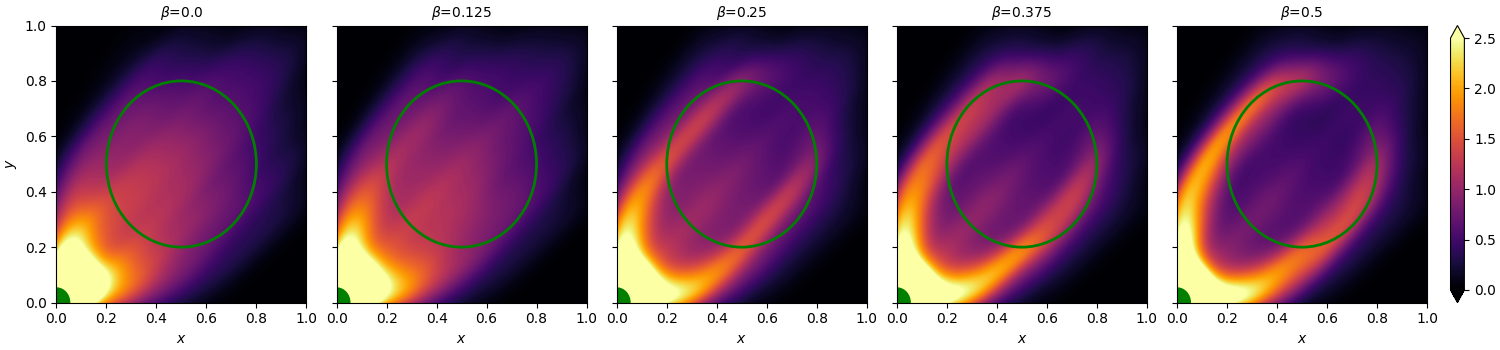}
	\caption{Position occurrence density heatmap of STAC on RiskyPointMass-v0 for different $\beta$ values. The green circle on the middle represents boundary of danger zone. Episodes are terminated when point mass enters green area on the lower left. }
	\label{fig:pointmass_heatmap}
\end{figure}

\paragraph{Environment Details} Following \citet{ma2021conservative} and \citet{ma2025dsac}, risky navigation task is used as a simple stochastic problem. The environment is named as \verb~RiskyPointMass-v0~. The goal of this task is to navigate a point mass towards a target point while avoiding a danger zone, in a two dimensional space. Danger zone is a circular area with a radius of $r_0 = 0.3$, centered at point $(0.5, 0.5)$. Initial states are randomly selected from $\mathcal{U}(0.3, 1)$ for both coordinates, excluding the danger zone. Episode terminates if the agent is close to target point $(0, 0)$ by Euclidean distance of $0.05$. At each step, the agent receives the sum of three reward components: negative Euclidean distance to the target, fixed $-0.1$ points to promote episode termination, a penalty of $-10$ points if the agent is within danger zone with probability $0.1\exp(-4 r^2/r_0^2)$ where $r$ is the distance to danger zone center. The agent can move by $0.1$ unit at most in both dimensions plus a random slip (at most $0.02$ unit). The reward function is designed so that optimal risk-neutral solution is to navigate linearly to the target point whereas optimal risk-averse solution is to navigate around the danger zone circle. 

\paragraph{Results} For each $\beta$ value, STAC is run with a single fixed seed, as this experiment is intended to provide a qualitative illustration of risk-sensitive behavior rather than a performance comparison. In Figure \ref{fig:pointmass_heatmap}, for varying $\beta$ values, position densities of the trained agents are demonstrated. The bigger green circles represent danger zone boundaries, whereas quarter green circles on the lower left represent the target points. For this demonstration, 500 evaluation steps are run with random initial points. Clearly, with higher $\beta$ values, the agent prefers to go around the circle as opposed to agents trained with lower $\beta$. This is an expected effect since pessimism is bootstrapped through TD targets, effectively inducing conservative behavior over the roll-out. To summarize, STAC not only mitigates overestimation but also induces risk-averse learning with pessimism on temporal uncertainty. 

\section{Related Work}

\paragraph{Overestimation reduction by epistemic uncertainty} In the literature, overestimation problem is solved by pessimistic learning based on \textit{epistemic uncertainty} of critic \citep{chen2021randomized, hiraoka2021dropout}. The same approach is also used in model-based RL methods \citep{janner2019trust, chua2018deep, depeweg2016learning}. Recently, for off-policy model-free actor-critic algorithms, epistemic uncertainty is estimated by using double networks \citep{fujimoto2018addressing, haarnoja2018soft}, or ensemble networks \citep{chen2021randomized, moskovitz2021tactical}, or dropout \citep{hiraoka2021dropout}. Ensemble of critics are computationally expensive as there are multiple of parameters to be optimized. The mentioned methods use a constant level of pessimism for policy evaluation and improvement, while \citet{moskovitz2021tactical} focused on updating pessimism \textit{on the fly} as a bandit problem rather than fixing it. Similarly \citet{cetin2023learning} tunes pessimism online by treating it as a dual variable  by enforcing the expected off-policy action-value bias to be zero. 

\paragraph{Overestimation reduction by aleatoric uncertainty}  
\citet{kuznetsov2020controlling} claims \textit{aleatoric uncertainty} is also responsible for overestimation since any randomness is exploited when the Bellman optimality operator ($\mathcal{T}^*$) is employed. For this purpose, they use ensemble networks for epistemic uncertainty, in which each network is a distributional network (modeled as quantiles) for aleatoric uncertainty modeling, and used both types of uncertainties for overestimation correction. Similarly, STAC uses distributional critic, but only models temporal (one-step) aleatoric uncertainty rather than full return distribution as this is the main source of overestimation, without epistemic modeling. 

\paragraph{Risk-sensitive reinforcement learning}  
Aleatoric uncertainty representation for the value function carries fundamental importance, especially in the presence of approximation \citep{bellemare2017distributional}. This can be conducted by atoms \citep{bellemare2017distributional}, quantiles \citep{dabney2018distributional, ma2025dsac}, or a normal distribution \citep{tang2019worst, yang2021wcsac} to model cumulative return behavior. For safety-critical RL applications to avoid catastrophic situations, aleatoric pessimism is used \citep{tang2019worst,yang2021wcsac,lim2022distributional,greenberg2022efficient,ma2025dsac}. These methods use pessimism only on actor update to scale risk sensitivity by modeling full return uncertainty, whereas STAC scales both actor and critic by pessimistic updates on temporal return uncertainty for overestimation mitigation. 

\section{Conclusion \& Future Directions}

In this paper, we introduced Stochastic Actor-Critic (STAC), a novel off-policy actor-critic algorithm. The main idea is to mitigate overestimation for the sake of faster and more robust learning by incorporating the pessimistic learning objective using temporal aleatoric critic uncertainty. For this purpose, the critic is modeled as a distributional neural network. Although normal distribution is used in STAC, our analysis is valid for all sub-Gaussian critic representations. 

We derived an upper bound for overestimation, demonstrating that an adequate level of pessimism mitigates overestimation without succumbing to underestimation, thus facilitating computational and sample-efficient learning. As a by-product, this approach also yields risk-averse learning process. Lastly, dropout on the actor and critic networks is proposed to mitigate overfitting. Actor dropout mostly improves performance, whereas using dropout on critic is only required to ensure learning convergence. 

\paragraph{Adaptive Pessimism and Regularization} While STAC demonstrates that temporal aleatoric pessimism alone is sufficient to mitigate overestimation using a single critic, several directions remain open. In particular, the pessimism coefficient $\beta$ is currently selected empirically and its optimal value varies across environments, motivating future work on adaptive or principled tuning strategies, similar to \citet{moskovitz2021tactical, cetin2023learning}. In addition, although STAC reduces computational overhead by avoiding double critics, a single-critic setup may require explicit regularization (e.g., dropout) for stable optimization in some environments. For future work, alternative regularization approaches might be considered for more stable training, including ensembles. 

\paragraph{Distribution Modeling} The effect of pessimism on highly stochastic environments is also an important topic for research. As shown, pessimistic updates lead to risk-averse behavior. Our results are limited to simple stochastic tasks, and it is worth investigating STAC empirically on more difficult ones. For more advanced modeling, aleatoric critic uncertainty can be modeled by quantiles \citep{dabney2018distributional, ma2025dsac} or flow matching \citep{chen2025unleashing, dong2025value}, rather than normal distribution. 

\paragraph{Exploration} Lastly, current methods use epistemic critic uncertainty for overestimation problem, which contradicts with \textit{optimism in the face of uncertainty} principle. Future researches might consider using epistemic critic and actor uncertainty for exploration, retaining temporal aleatoric pessimism. 

\paragraph{Broader Impact} STAC is an important contribution to understand overestimation phenomenon in off-policy actor-critic learning. It demonstrates that critic overestimation is also a topic of distributional learning, and the solution is to devise a risk-averse objective for both actor and critic. Moreover, it also proves that modeling epistemic uncertainty is not necessary at all for this problem, improving computational efficiency directly. 

\subsubsection*{Acknowledgments}
This research has received no external funding. 


\bibliographystyle{unsrtnat}
\bibliography{references}  

\newpage 

\begin{appendices}
	
	\section{Proofs}
	\label{app:proofs}
	
	\begin{proof}[Proof of Theorem \ref{thm:overestimation}]
		Analyzing Bellman update $\mathcal{T}^{*} Q(s,a)$, 
		\begin{align*}
			\mathcal{T}^{*} Q(s,a) &= R(s,a) + \gamma \mathbb{E}_{Q\sim \mathcal{Q}} \Big[ \Tilde{\alpha} \log \Big( \int\limits_{\mathcal{A}} \exp({\Tilde{\alpha}}^{-1} Q(s',a')) da' \Big) \Big] \\
			&\leq R(s,a) + \gamma \Tilde{\alpha} \log \Big( \mathbb{E}_{Q\sim \mathcal{Q}} \Big[ \int\limits_{\mathcal{A}} \exp({\Tilde{\alpha}}^{-1} Q(s',a')) da'  \Big] \Big) \\
			&= R(s,a) + \gamma \Tilde{\alpha} \log \Big( \int\limits_{\mathcal{A}} \mathbb{E}_{Q\sim \mathcal{Q}} \Big[ \exp({\Tilde{\alpha}}^{-1} Q(s',a')) \Big] da' \Big) \\
			&\leq R(s,a) + \gamma \Tilde{\alpha} \log \Big( \int\limits_{\mathcal{A}} \exp( {\Tilde{\alpha}}^{-1} \mu(s',a') + \frac{1}{2} {\Tilde{\alpha}}^{-2} \sigma^2(s',a') ) da' \Big) \\
		\end{align*}
		First inequality is by Jensen's inequality (using concave property of $\log$ function) while the following equality is a result of Tonelli's theorem. The second inequality results from the sub-Gaussian assumption \ref{defn:subgauss}. 
	\end{proof} 
	
	\begin{proof}[Proof of Corollary \ref{cor:pessimism}]
		From the Theorem \ref{thm:overestimation}, we can show that 
		\begin{align*}
			\mathcal{T}_{\beta}^{*} Q(s,a) 
			&= R(s,a) + \gamma \mathbb{E}_{Q\sim \mathcal{Q}} \Big[ \Tilde{\alpha} \log \Big( \int\limits_{\mathcal{A}} \exp({\Tilde{\alpha}}^{-1} Q(s',a') - \beta\sigma(s',a')) da' \Big) \Big] \\
			&\leq R(s,a) + \gamma \Tilde{\alpha} \log \Big( \int\limits_{\mathcal{A}} \exp( {\Tilde{\alpha}}^{-1} (\mu(s',a') - \beta \sigma(s',a') + \frac{1}{2} {\Tilde{\alpha}}^{-1} \sigma^2(s',a')) ) da' \Big) \\
			&=  R(s,a) + \gamma \Tilde{\alpha} \log \Big( \int\limits_{\mathcal{A}} \exp( {\Tilde{\alpha}}^{-1} \mu^{\dagger}(s',a')) da' \Big) = \mathcal{T}^{*} \mu^{\dagger}(s,a). 
		\end{align*}
		where we have defined $\mu^{\dagger}(s',a') = \mu(s',a') - \beta \sigma(s',a') + \frac{1}{2} {\Tilde{\alpha}}^{-1} \sigma^2(s',a') $. If $\beta \geq \max\limits_{(s', a')} \frac{1}{2}{\Tilde{\alpha}}^{-1} \sigma(s', a')$, then $\mu^{\dagger}(s',a') < \mu(s',a')$. So we can show that
		\begin{equation}
			\mathcal{T}_{\beta}^{*} Q(s,a) \leq \mathcal{T}^{*} \mu^{\dagger}(s,a) \leq \mathcal{T}^{*} \mu(s,a).  
		\end{equation}
	\end{proof}
	
	\begin{proof}[Proof of Theorem \ref{thm:overestimation_bound}]
		From the Theorem \ref{thm:overestimation}, we can show that
		\begin{align*}
			\mathbb{E}_{s'\sim\tau} [ \mathcal{T}^{*} Q(s,a) ] &\leq R(s,a) + \gamma \mathbb{E}_{s'\sim\tau } \Big[ \Tilde{\alpha} \log \Big( \int\limits_{\mathcal{A}} \exp( {\Tilde{\alpha}}^{-1} \mu(s',a') + \frac{1}{2} {\Tilde{\alpha}}^{-2} \sigma^2(s',a') ) da' \Big) \Big] \\
			&\leq R(s,a) + \gamma \mathbb{E}_{s'\sim\tau } \Big[ \Tilde{\alpha} \log \Big( \big( \int\limits_{\mathcal{A}} \exp( {\Tilde{\alpha}}^{-1} \mu(s',a') ) da' \big) \cdot \big( \max_{a'} \exp( \frac{1}{2} {\Tilde{\alpha}}^{-2} \sigma^2(s',a') \big) \Big) \Big] \\  
			&= R(s,a) + \gamma \mathbb{E}_{s'\sim\tau } \Big[ \Tilde{\alpha} \log \Big( \int\limits_{\mathcal{A}} \exp( {\Tilde{\alpha}}^{-1} \mu(s',a') ) da' \Big) \Big] + \frac{\gamma}{2\Tilde{\alpha}} \mathbb{E}_{s'\sim\tau } \Big[ \max_{a'} \sigma^2(s',a') \Big]. \\ 
		\end{align*}
		
		The second inequality is a result of the mean value theorem for integrals. In the last equality, the first two terms are equal to $\mathcal{T}^{*} \mu(s,a)$. Therefore, 
		\begin{align*}
			\epsilon(s, a) = \mathbb{E}_{s'\sim\tau} [ \mathcal{T}^{*} Q(s,a) ] - \mathcal{T}^{*} \mu(s,a) \leq \frac{\gamma}{2\Tilde{\alpha}} \mathbb{E}_{s'\sim\tau} \Big[ \max_{a'} \sigma^2(s', a') \Big]. 
		\end{align*} 
	\end{proof}
	
	\newpage
	
	\section{Results of Ablation Studies}
	\label{app:results_of_ablations}
	
	\subsection{Pessimism Ablation}
	\label{app:pessimism_ablation}
	
	\begin{figure}[H]
		\centering
		\includegraphics[width=0.95\textwidth,height=\textheight,keepaspectratio]{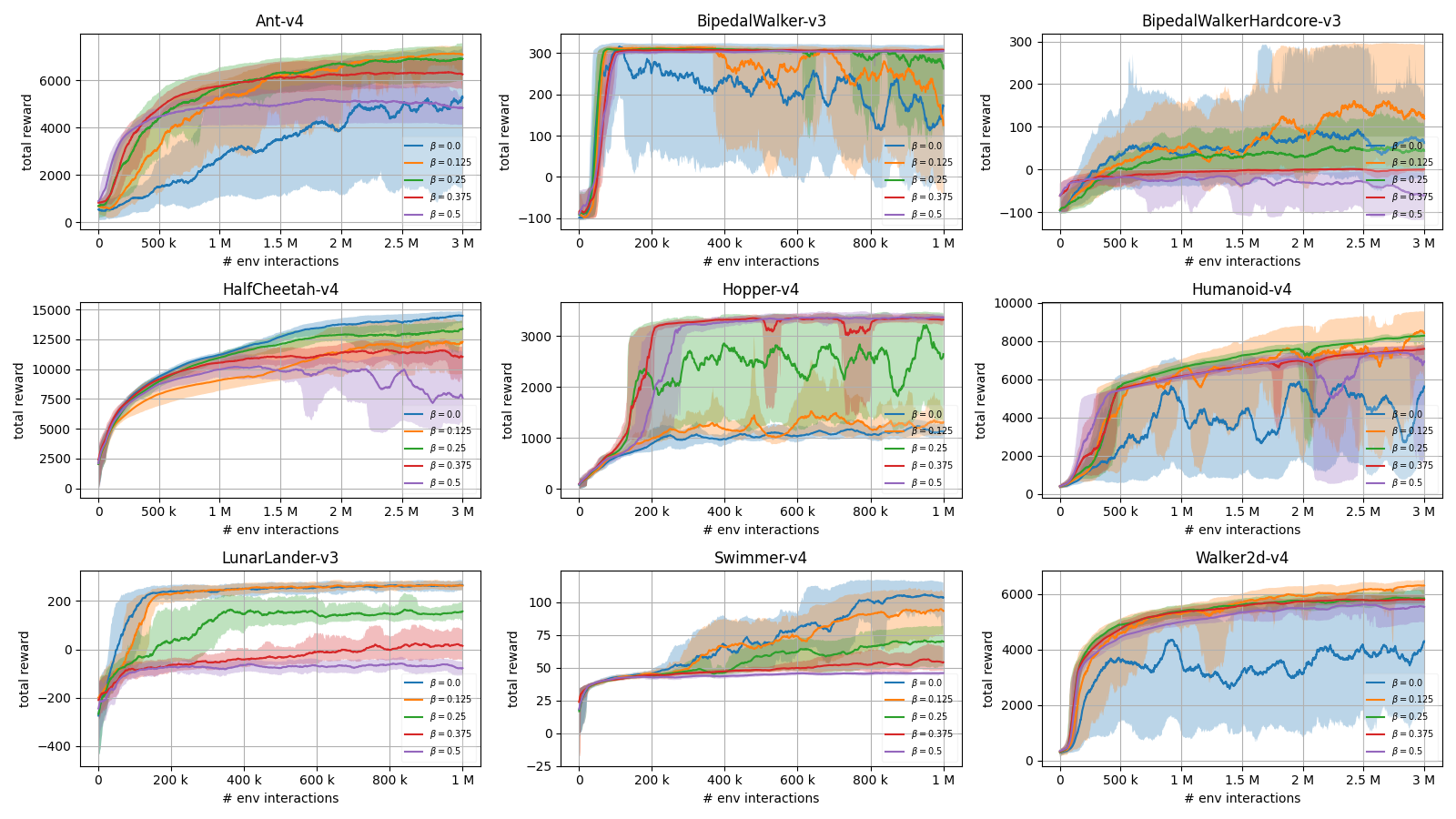}
		\caption{Episodic score curves of STAC with varying pessimism ($\beta$) parameter, with actor and critic dropout equal to $0.01$. }
		\label{fig:pess_sens_pol_cri_dropout}
	\end{figure}
	
	\begin{figure}[H]
		\centering
		\includegraphics[width=0.95\textwidth,height=\textheight,keepaspectratio]{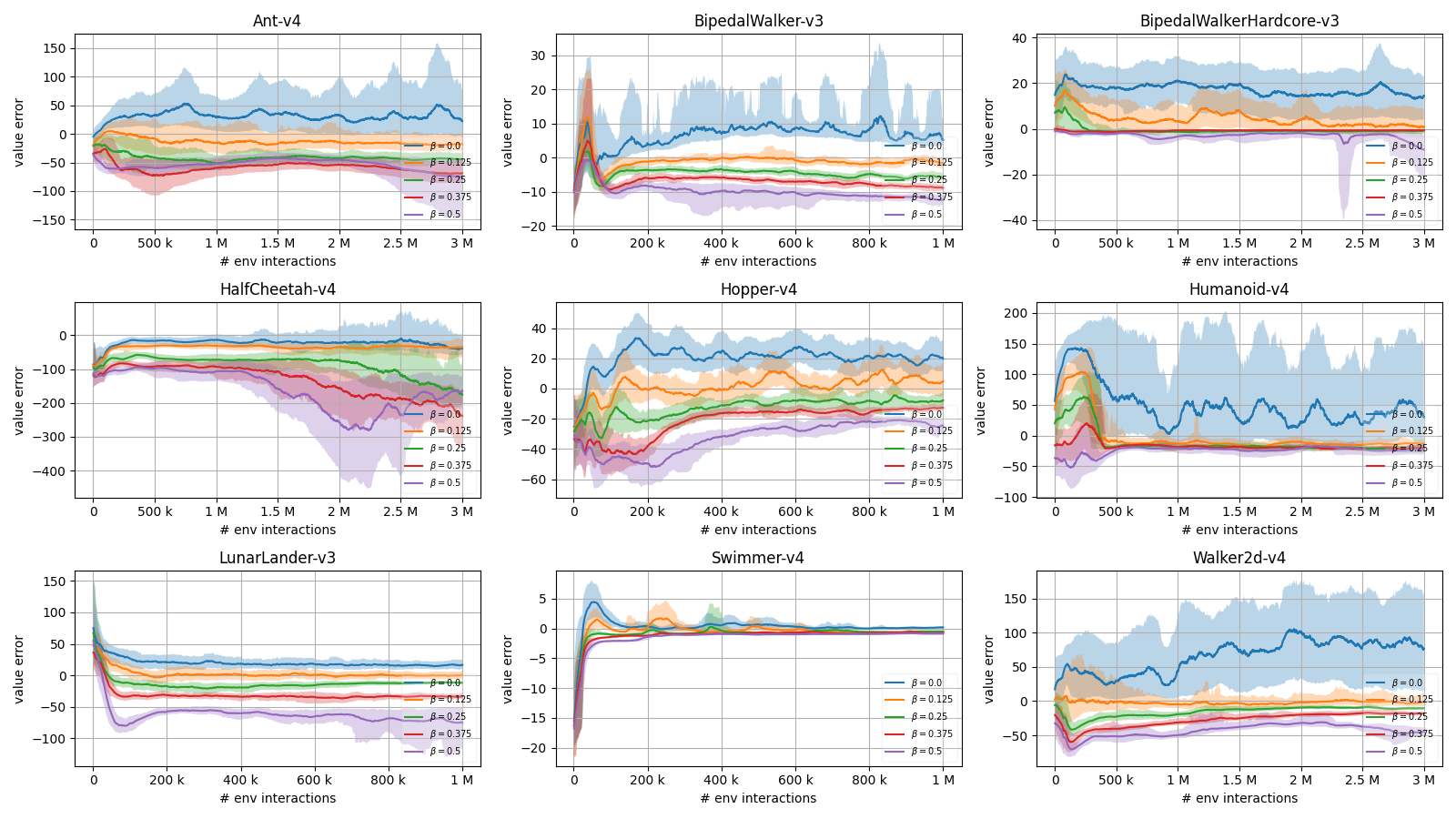}
		\caption{Average episodic value estimation error curves of STAC with varying pessimism ($\beta$) parameter, with actor and critic dropout equal to $0.01$. }
		\label{fig:pess_sens_error_pol_cri_dropout}
	\end{figure}	
	
	\subsection{Dropout Configuration Ablation}
	\label{app:dropout_ablation}
	
	\begin{figure}[H]
		\centering
		\includegraphics[width=0.95\textwidth,height=\textheight,keepaspectratio]{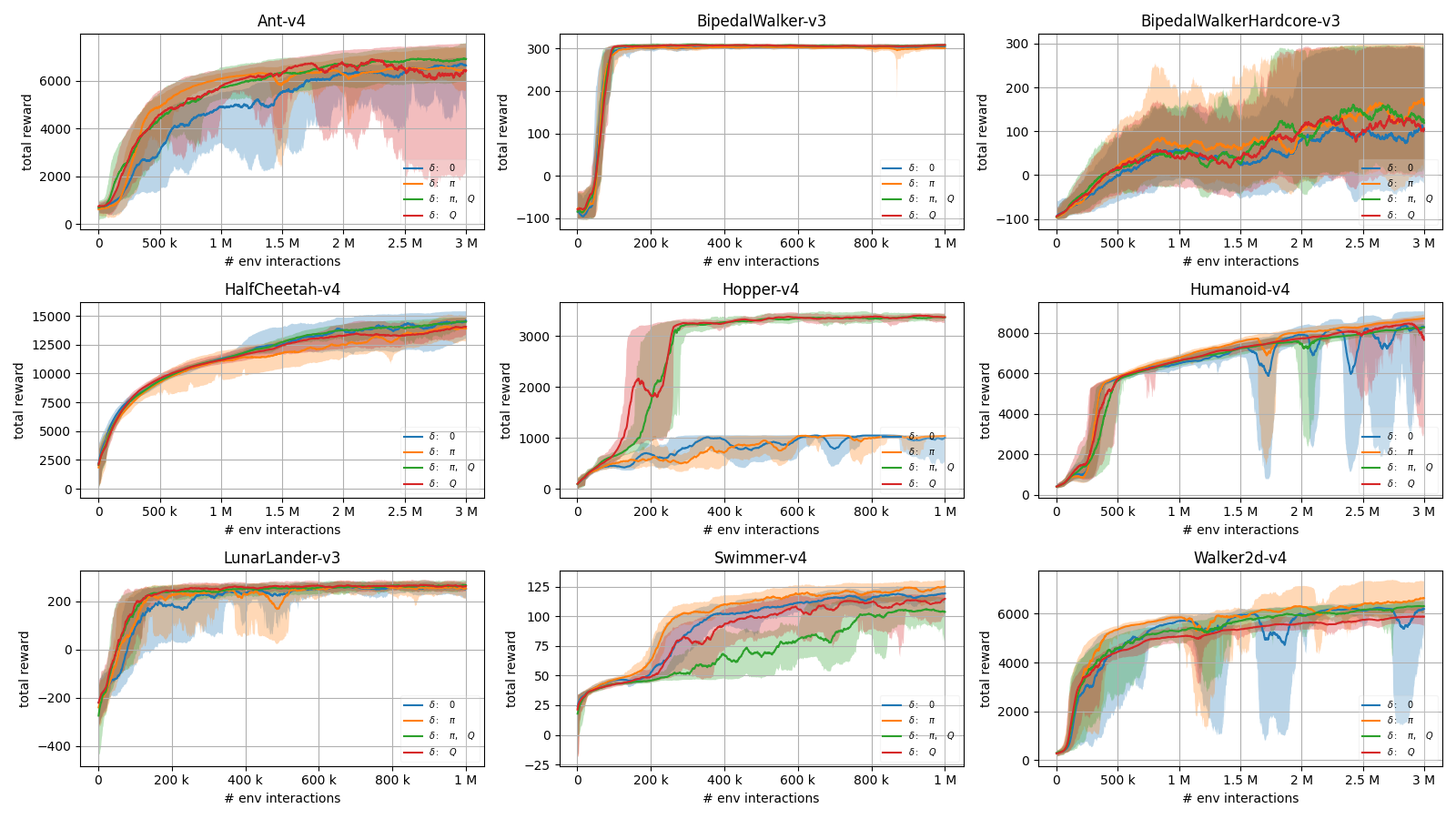}
		\caption{Episodic score curves of STAC with 4 different dropout configurations, under fixed pessimism level yielding best performance (see Table \ref{tab:target_ent_pessimism_algos}). }
		\label{fig:drop_sens}
	\end{figure}
	
	\begin{figure}[H]
		\centering
		\includegraphics[width=0.95\textwidth,height=\textheight,keepaspectratio]{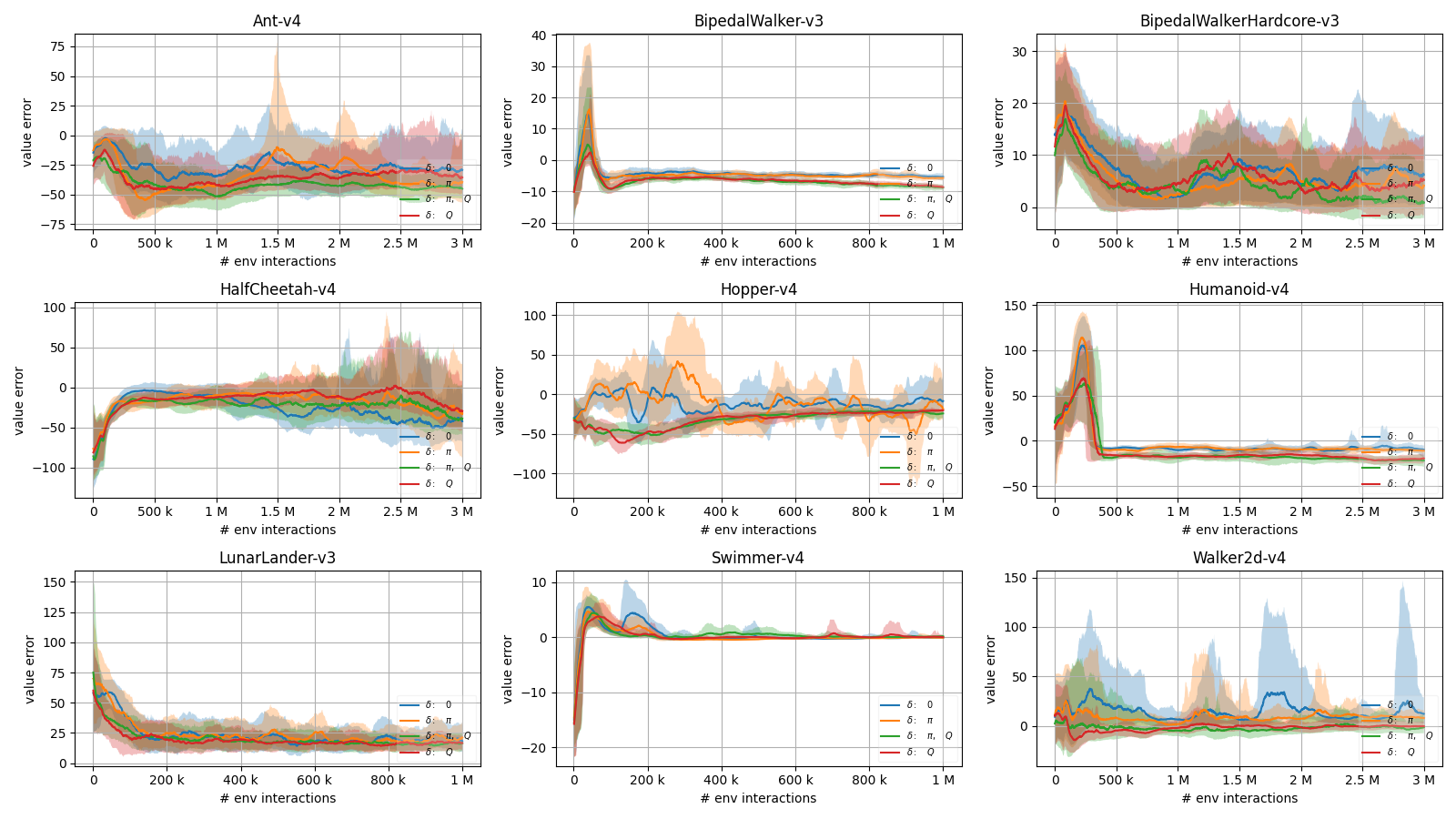}
		\caption{Average episodic value estimation error curves of STAC with 4 different dropout configurations, under fixed pessimism level yielding best performance (see Table \ref{tab:target_ent_pessimism_algos}). }
		\label{fig:drop_sens_error}
	\end{figure}

	\section{Hyper-parameters and Experiment Details}
	\label{app:hyperparam_exp}
	
	
	\begin{table}[H]
		\centering
		\caption{Experimental Parameters per Algorithm}
		\label{tab:parameters_algos}
		
		\begin{tabular}{@{}lll@{}}
			\toprule
			\textbf{Algorithm}             & \textbf{Parameter}      & \textbf{Value}          \\ \midrule
			\multirow{11}{*}{DSAC, ESTAC, SAC, STAC}    & Optimizer               & Adam (\citep{kingma2014adam}) \\
			& Critic Learning Rate    & $3 \times 10^{-4}$      \\
			& Actor Learning Rate     & $3 \times 10^{-4}$      \\
			& Discount Rate ($\gamma$)& 0.99                    \\
			& Target-Smoothing Coefficient ($\rho$) & 0.995       \\
			& Target Entropy ($\bar{H}$) & $-|\mathcal{A}|$       \\			
			& Replay Buffer Size      & $1 \times 10^{6}$       \\
			& Mini-Batch Size         & 256                     \\
			& Learning Starting Step    & 10000                   \\
			& UTD Ratio ($G$)         & 1                        \\
			\midrule
			\multirow{2}{*}{DSAC}     & Number of Quantiles     & 25 \\
			& Huber Loss Threshold ($\kappa$)           & 1.0                     \\
			\bottomrule
		\end{tabular}
	\end{table}
	
	\begin{table}[H]
		\centering
		\caption{Target policy entropy ($\bar{H}$), pessimism $\beta$ and dropout rates per environment, yielding best results. All algorithms share same policy entropy. }
		\label{tab:target_ent_pessimism_algos}
		\begin{tabular}{@{}lllll@{}}
			\toprule
			\textbf{Environment}   & \textbf{Entropy ($\bar{H}$)} & \textbf{Pessimism ($\beta$)} & \textbf{Actor Dropout} & \textbf{Critic Dropout} \\ \midrule
			\verb~Ant-v4~                      & -4        & 0.25      & 0.01      &  0.01     \\
			\verb~BipedalWalker-v3~            & -2        & 0.375     & 0         &  0.01     \\
			\verb~BipedalWalkerHardcore-v3~    & -2        & 0.125     & 0.01      &  0    \\						
			\verb~HalfCheetah-v4~              & -3        & 0.0       & 0         &  0      \\			
			\verb~Hopper-v4~                   & -1        & 0.5       & 0.01      &  0.01     \\
			\verb~Humanoid-v4~                 & -8        & 0.25      & 0.01      &  0    \\
			\verb~LunarLander-v3~   		   & -2        & 0.0       & 0.01      &  0.01     \\ 
			\verb~Swimmer-v4~                  & -1        & 0.0       & 0.01      &  0     \\	
			\verb~Walker2d-v4~                 & -3        & 0.125     & 0.01      &  0     \\								
			\bottomrule
		\end{tabular}
	\end{table}	
	
	\section{Network Architectures of STAC}
	\label{app:network_archs}
	
	\begin{figure}[H]
		\centering
		\begin{minipage}{.5\textwidth}
			\centering
			
			\begin{tikzpicture}[node distance=0.2cm, baseline]
				
				\node[inplayer] (input) {Input ($d_s + d_a$): $(s,a)$};
				
				\node[layer,below=of input] (fc1) {Linear (256)};
				\node[dropout,below=0.2cm of fc1] (dropout1) {Dropout};
				\node[norm,below=0.2cm of dropout1] (norm1) {LayerNorm};
				\node[activation,below=0.2cm of norm1] (relu1) {ReLU};
				
				\node[layer,below=of relu1] (fc2) {Linear (256)};
				\node[dropout,below=0.2cm of fc2] (dropout2) {Dropout};
				\node[norm,below=0.2cm of dropout2] (norm2) {LayerNorm};
				\node[activation,below=0.2cm of norm2] (relu2) {ReLU};
				
				\node[layer,below=of relu2] (output) {Linear (2)};
				\node[output,below=0.2cm of output] (gaussian) {$\mathcal{N} (\mu(1), \sigma^2(1))$};

				\draw[->] (input) -- (fc1);
				\draw[->] (fc1) -- (dropout1);
				\draw[->] (dropout1) -- (norm1);
				\draw[->] (norm1) -- (relu1);
				\draw[->] (relu1) -- (fc2);
				\draw[->] (fc2) -- (dropout2);
				\draw[->] (dropout2) -- (norm2);
				\draw[->] (norm2) -- (relu2);
				\draw[->] (relu2) -- (output);
				\draw[->] (output) -- (gaussian);
				
				
			\end{tikzpicture}
			
			\captionof{figure}{Critic network architecture}
			\label{fig:critic_visual}
		\end{minipage}%
		\begin{minipage}{.5\textwidth}
			\centering
			
			\begin{tikzpicture}[node distance=0.2cm]
				
				\node[inplayer] (input) {Input ($d_s$): $s$};
				
				\node[layer,below=of input] (fc1) {Linear (256)};
				\node[dropout,below=0.2cm of fc1] (dropout1) {Dropout};
				\node[norm,below=0.2cm of dropout1] (norm1) {LayerNorm};
				\node[activation,below=0.2cm of norm1] (relu1) {ReLU};
				
				\node[layer,below=of relu1] (fc2) {Linear (256)};
				\node[dropout,below=0.2cm of fc2] (dropout2) {Dropout};
				\node[norm,below=0.2cm of dropout2] (norm2) {LayerNorm};
				\node[activation,below=0.2cm of norm2] (relu2) {ReLU};
				
				\node[layer,below=of relu2] (output) {Linear (2$d_a$)};
				\node[output,below=0.2cm of output] (squashed) {$\tanh_{\#}\mathcal{N} ( \mu(d_a), \sigma^2(d_a) )$};
				
				\draw[->] (input) -- (fc1);
				\draw[->] (fc1) -- (dropout1);
				\draw[->] (dropout1) -- (norm1);
				\draw[->] (norm1) -- (relu1);
				\draw[->] (relu1) -- (fc2);
				\draw[->] (fc2) -- (dropout2);
				\draw[->] (dropout2) -- (norm2);
				\draw[->] (norm2) -- (relu2);
				\draw[->] (relu2) -- (output);
				\draw[->] (output) -- (squashed);
				
			\end{tikzpicture}
			
			\captionof{figure}{Actor network architecture}
			\label{fig:policy_visual}
		\end{minipage}
	\end{figure}

	\section{Source Code}
	\label{app:code}
	Our results can be accessed publicly at \url{https://github.com/ugurcanozalp/stochastic-actor-critic}. This code uses our in-house developed RL framework as a sub-repository, available on \url{https://github.com/ugurcanozalp/rl-warehouse}. 
	
\end{appendices}

\end{document}